\documentclass[sigconf]{acmart}
\setkeys{acmart.cls}{balance=false}

\acmYear{2026}
\acmConference[DAI '26]
  {The Eighth International Conference on Distributed Artificial Intelligence}
  {November 29--December 2, 2026}
  {Hong Kong, China}
\acmBooktitle{The Eighth International Conference on Distributed Artificial
  Intelligence (DAI '26), November 29--December 2, 2026, Hong Kong, China}

\setcopyright{none}
\copyrightyear{}
\acmDOI{}
\acmISBN{}

\usepackage{array}
\usepackage{booktabs}
\usepackage{float}
\newcommand{\para}[1]{\par\noindent\textbf{#1.}}
\newtheorem{proposition}{Proposition}
\hypersetup{pdfauthor={Lifei Liu; Haoran Yu; Danping Zhang}}

\begin{document}

\title{Revocation Is Not Complete at Receipt: An Executable Outcome Contract for Distributed Agent-Tool Gateways}

\author{Lifei Liu}
\affiliation{%
  \institution{Independent Researcher}
  \city{Seattle}
  \state{WA}
  \country{USA}}
\email{lliu.lifei@gmail.com}

\author{Haoran Yu}
\affiliation{%
  \institution{Independent Researcher}
  \city{Seattle}
  \state{WA}
  \country{USA}}
\email{haoranyu889@gmail.com}

\author{Danping Zhang}
\affiliation{%
  \institution{Nanchang Hangkong University}
  \city{Nanchang}
  \country{China}}
\email{zhangdanping@nchu.edu.cn}

\renewcommand{\shortauthors}{Liu et al.}

\begin{abstract}
Distributed agent workflows cache delegated authority at tool gateways.
Revocation may be reported at issuer commit, transport acknowledgment, or
receiver persistence, although none of these observations establishes that all
gateways enforced the change or that guarded effects stopped.

We define C0--C5, a bounded outcome contract spanning issuer commit through
effect prevention, and implement it as a persistent monitor over a
frozen gateway set, call set, and observation window. A 14-case fault matrix
checks conservative classification, sticky violations, restart persistence,
and evidence immutability. Treating C3 or C4 as completion produces seven or
five false positives in a derived eight-case audit; the bounded C0--C5 verdict
produces none.

Across two separately provisioned 16-host Keycloak executions, normal delivery
and recovery commit no effects, whereas receiver-selective non-delivery admits
and commits 112 of 225 calls (49.78\%) in every treatment block despite
successful transport and receiver persistence. OPA/Toxiproxy, a live networked
MCP workflow, and an unmodified closed-loop MCPMark agent reproduce the
persisted-but-not-enforced boundary on narrower surfaces. In the MCPMark
execution, normal C4 denies all 31 exposed post-barrier calls with no effects,
whereas selective C3 allows 344 calls and commits 145 effects. Transport
success is therefore evidence of progress, not a revocation-completion verdict.
\end{abstract}


\maketitle

\section{Introduction}
\label{sec:intro}

Consider an agent workflow whose delegated authority is cached at eight tool
gateways. A security-event transmitter sends a signed revocation to two durable
receivers, each responsible for four gateways. Both receivers acknowledge and
persist the event, but one receiver cannot forward it across its local control
path. From the transmitter's view, delivery succeeded; from the receiver's
view, the event is durable. The four downstream gateways nevertheless retain
the stale authority and can still commit Filesystem and PostgreSQL effects.
This is the failure isolated by E-B24: in every treatment block, receiver B
returns HTTP 202 and stores the event while its four gateways admit all 112
calls assigned to them. The example exposes the question at the center of this
paper: when may a distributed agent system declare revocation complete?

Prior work provides substantial machinery for moving and recording revocation
state. Agent-authorization proposals specify signed notifications, retries,
status fallback, descendant invalidation, authoritative logs, heartbeat expiry,
gossip, and evidence records
~\cite{singla2026aip,cruz2026aap,kumar2026daap,nelson2026drp,deochake2026hbhc}.
OAuth token revocation already recognizes that some servers may learn an
invalidation later than others~\cite{rfc7009}.
Global Token Revocation distinguishes a successful revocation request from
revocation completion and permits an optional completion notification
~\cite{parecki2026globaltokenrevocation}. Its completion condition concerns
issuer-side token revocation and reauthentication; it does not require receipts
from a distributed gateway population or effect-side observations.
CAEP and the Shared Signals Framework (SSF) define security events and PUSH or
POLL delivery between transmitters and receivers
~\cite{openid2025caep,openid2025ssf}; permit receipts move authorization
evidence closer to an effect boundary~\cite{lee2026permit}. Proof of Execution
separates authorization, enforcement, effects, and records, while Notarized
Agents obtains receiver-attested action receipts
~\cite{rhodes2026proofexecution,figuera2026notarized}. These systems address
complementary objects. Issuer commit, transport acknowledgment, receiver
persistence, gateway enforcement, and effect prevention remain distinct facts.
Evaluating only the earlier stages can hide a stale gateway; evaluating only
denied calls can hide fail-closed availability loss; and a modeled propagation
result does not establish that real tool effects stopped.

We study revocation completion as an end-to-end, bounded outcome rather than as
a property of a particular propagation mechanism. The research question is:
\emph{what evidence is sufficient to certify that a declared set of agent-tool
gateways enforced a revocation and prevented a declared set of effects during
a finite observation window?} We separate the path into C0--C5: issuer commit,
transmitter queueing, transport acknowledgment, receiver persistence, gateway
enforcement, and effect prevention. The bound is explicit. It applies only to
the registered gateways, call vectors, and time window; it does not claim that
an unknown gateway or a future call is safe. This formulation also preserves
the value of C0--C3. They are evidence of progress, but an indistinguishability
argument shows that they cannot, without stronger assumptions, distinguish an
all-enforced execution from one in which a guarded effect commits.

Three obstacles make this outcome difficult to evaluate. First, evidence is
split across administrative and execution boundaries, so a completion decision
must join receiver records, gateway receipts, call decisions, and independent
effect observations without allowing later evidence to erase a violation.
Second, a causal test must interrupt only the receiver-to-gateway handoff while
leaving issuer-to-receiver delivery, data-plane calls, effect services, clocks,
and observation paths intact. Otherwise a failed effect could reflect general
outage rather than revocation enforcement. Third, a result obtained from one
self-authored gateway or one synthetic trace may be implementation-specific.
The evaluation must therefore change both the enforcement stack and the agent
execution surface while retaining unfavorable mechanism-cost and workload
results as claim boundaries.

We address these obstacles with an executable contract and a staged evidence
ladder. A persistent C0--C5 monitor freezes the gateway set, call set, and
observation window, rejects conflicting evidence, and returns bounded complete,
partial, indeterminate, or violated verdicts; E-B27 tests its classification and
integrity behavior. E-B24 then isolates post-persistence non-delivery on two
separately provisioned 16-host deployments and judges outcomes with independent
Filesystem and PostgreSQL oracles. Both 30-block executions yield the same
topology-predicted result: normal delivery and recovery commit no effects,
whereas the fault admits and commits exactly 112 of 225 calls (49.78\%) in every
block. E-B25 changes the enforcement and fault stack to OPA and Toxiproxy, and
E-B26 executes a live three-role LLM workflow over networked MCP tools. E-B28
then removes frozen action replay: an unmodified MCPMark agent observes each
tool result and chooses its next call across the ten frozen E-B20 targets. In
60 valid task arms, selective C3 produces 145 post-barrier effects while normal
C4 produces none. E-B17 and E-B20 retain, respectively, the 50\%
mechanism-cost failure and the failed external-workload hypothesis rather than
converting them into supporting wins.

We make three contributions:
\begin{itemize}
  \item We formulate distributed agent revocation completion as a bounded
        C0--C5 outcome contract and realize it as an executable monitor with
        immutable scope, finite-window closure, sticky violations, and joined
        receipt/effect evidence. E-B27 checks 14 frozen fault cases; a derived
        detector audit exposes the false completions produced by C3- and
        C4-only rules and measures monitor costs through 1,024 gateways and
        10,000 vectors
        (Section~\ref{sec:contract} and Section~\ref{sec:monitor_evaluation}).
  \item We provide a primary multi-host causal measurement and a separately
        provisioned repeatability execution of the persisted-but-not-enforced
        gap. Each uses 16 AWS hosts, eight gateways in three zones, durable
        receivers, an external fault, and two effect oracles. Both measure the
        same 49.78 percentage-point treatment effect and zero-effect recovery
        (Sections~\ref{sec:design}--\ref{sec:results}).
  \item We test whether the boundary depends on one implementation or replay
        surface. E-B25 reproduces it with OPA 1.17 and Toxiproxy 2.12. E-B28
        runs an unmodified external agent in a closed loop over ten frozen
        Filesystem/PostgreSQL tasks: 344 selective post-barrier calls commit
        145 effects, whereas 31 exposed normal calls are all denied. E-B26 is
        a narrower live-workflow sanity check; E-B17's 50\% cost failure and
        E-B20's failed workload hypothesis delimit the claim
        (Sections~\ref{sec:triangulation}--\ref{sec:cross_surface} and
        Appendices~\ref{app:eb17}--\ref{app:supporting}).
\end{itemize}

The paper does not propose a new revocation primitive or attribute the gap to a
Keycloak defect. Its contribution is an evaluation contract and bounded
evidence for a completion gap at the distributed agent execution boundary.

\section{End-to-End Revocation Completion}
\label{sec:contract}

\subsection{Object, Stages, and Verdict}

We represent delegated tool authority as a signed edge
\(e=(i,s,t,\sigma,\tau)\), where issuer \(i\) authorizes subject agent \(s\)
to call tool \(t\) within scope \(\sigma\) until expiry \(\tau\). Let
\(G(e)\) be the set of gateways that may accept the edge. A revocation names
the edge, sequence number, effective time, and issuer. Gateways accept only
valid, nondecreasing state.

The completion contract separates six stages:
\begin{description}
  \item[C0 Issuer committed.] The authoritative issuer records the
        revocation.
  \item[C1 Queued.] The transmitter durably queues a signed event.
  \item[C2 Transport acknowledged.] PUSH returns success or a POLL item is
        acknowledged.
  \item[C3 Persisted.] A receiver verifies and durably stores the event.
  \item[C4 Gateway enforced.] Every assigned gateway durably installs deny
        state and emits a receipt binding the gateway, delegation, event,
        channel, and application time.
  \item[C5 Effect prevented.] For every assigned gateway \(g\), no guarded
        effect attributable to \(g\) commits in the declared post-revocation
        observation window.
\end{description}

For an edge and observation window \(W\), the execution-level verdict is
\[
 \mathrm{Complete}(e,W) = C0(e) \land
 \bigwedge_{g\in G(e)} C4(g,e) \land C5(g,e,W).
\]
C1--C3 record progress in asynchronous delivery but are insufficient for this
verdict. C4 requires evidence from the declared enforcement population, whereas
C5 requires a per-gateway effect-side observation independent of the gateway
response.

\subsection{Observation Lower Bound and Executable Monitor}

\begin{proposition}[C0--C3 indistinguishability]
Under post-persistence selective non-delivery, no completion detector whose
input ends at C3 can soundly decide the C0--C5 outcome contract.
\end{proposition}
\begin{proof}
Construct two executions with the same signed event and identical C0--C3
observations. In the first, every declared gateway installs deny state and all
bounded calls are denied. In the second, one receiver-to-gateway handoff fails,
that gateway retains cached authority, and one bounded call commits. A detector
restricted to the common prefix must return the same verdict for both. A
positive verdict is unsound for the second execution, while a negative verdict
cannot certify the first. Distinguishing them requires C4/C5 evidence or a
stronger synchrony or fail-closed assumption.
\end{proof}

We implement the contract in a persistent reference monitor (bounded-contract
version 2). Registration freezes the event, physical-gateway set, expected call
vectors, and finite observation window. The caller supplies nonempty C0--C2
evidence identifiers; the monitor does not independently attest the issuer or
transmitter. It verifies the corresponding persisted receiver event and joins
receiver-side gateway receipts with independently stored allow/deny decisions
and effect-commit observations. Identical duplicate evidence is idempotent,
whereas conflicting receipt or effect records are rejected. The monitor also
rejects attempts to close the window before its frozen deadline or at a future
timestamp.

The monitor returns bounded completion only after every expected gateway has a
receipt, every expected vector has an observation, the window has closed, and
no allowed call or committed effect has been recorded. Otherwise, it reports
partial enforcement, indeterminate effect evidence, or violation. A late
receipt may advance the recorded stage but cannot erase an observed violation.
The verdict does not claim that an unknown gateway or future call can never
produce an effect. Figure~\ref{fig:contract_path} maps this interface to the
E-B24 topology: the normal and fault arms share C0--C3, while the monitor joins
gateway receipts with the two independent effect oracles.

\begin{figure*}[t]
  \centering
  \includegraphics[width=0.98\textwidth]{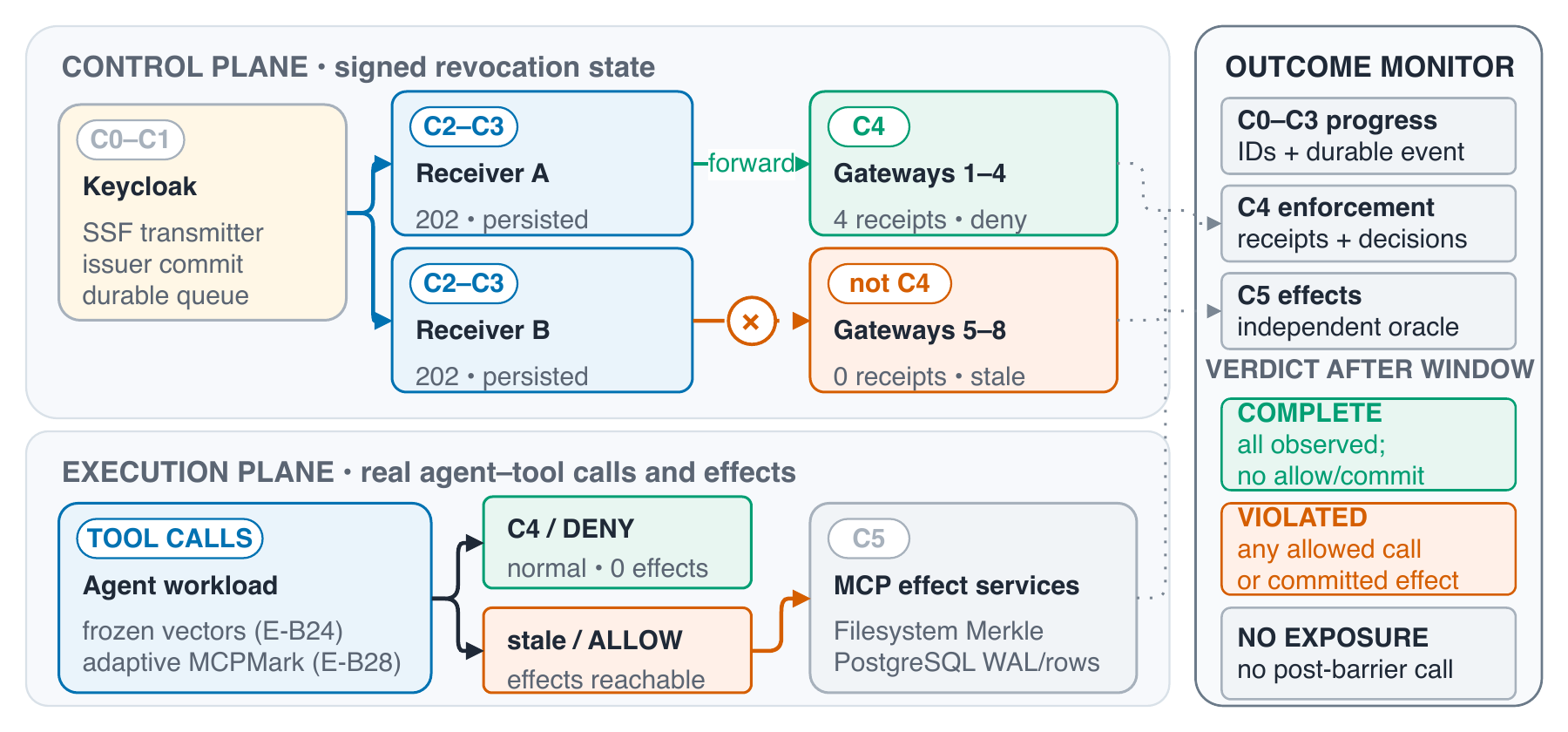}
  \caption{Revocation completion requires outcome evidence beyond delivery.
  Both paths reach C2/C3; only the healthy path reaches C4 and denies agent
  calls, whereas receiver-selective non-delivery (orange cross) leaves stale
  gateways able to reach Filesystem or PostgreSQL. The monitor joins C0--C5
  evidence over a frozen bound; dotted lines carry evidence, not control.
  E-B24 uses frozen vectors, while E-B28 uses adaptive MCPMark calls.}
  \Description{A two-layer architecture separates the signed revocation
  control plane from real agent-tool execution. Keycloak sends an event to two
  durable receivers. Receiver A reaches gateways one through four and installs
  deny state. An orange cross interrupts receiver B's handoff to gateways five
  through eight, which retain stale authority. Agent calls are denied on the
  healthy path but can reach Filesystem and PostgreSQL effect services on the
  stale path. A monitor combines C0 through C5 evidence and returns complete,
  violated, or no-exposure verdicts.}
  \label{fig:contract_path}
\end{figure*}

\subsection{Selective Non-Delivery}

Selective non-delivery makes a valid positive revocation unavailable to a
chosen subset while other recipients and paths continue normally. E-B17
suppresses an operator-owned path to selected gateways but leaves independent
repair reachable. E-B24 tests a narrower execution fault: receiver B accepts
and stores the event, while an external network rule blocks only its handoff to
the gateways it owns. Keycloak-to-receiver delivery, driver-to-gateway calls,
gateway-to-effect traffic, clocks, and the corresponding receiver A paths
remain available.

The adversary cannot forge signatures, alter sequence numbers, modify gateway
or workload code by experimental arm, or suppress all observation. Permanent
isolation remains an impossibility boundary: no asynchronous protocol can
force an unreachable gateway to learn new state. The measurement question is
whether visible C2/C3 progress can conceal that isolation from an observer that
declares completion.

\subsection{Stage Completion and Mechanism Feasibility}

The stage contract answers whether a concrete revocation reached its effect
boundary. Comparing mechanisms requires additional outcomes. E-B17 therefore
judges explicit-event completeness, revoked calls admitted before enforcement,
benign availability loss, independently checkable audit evidence, and control
messages/bytes together. A mechanism is feasible only if every frozen gate is
met. This second layer prevents a low stale-call count from hiding fail-closed
availability loss or an unaffordable repair policy. C0--C5 and the feasibility
outcomes serve different purposes: the former prevents false completion, while
the latter prevents one favorable metric from becoming a mechanism verdict.

\section{Evaluation Design}
\label{sec:design}

\subsection{Evidence Ladder}

Table~\ref{tab:evidence_ladder} assigns a distinct role to each major study. The
causal claim comes from E-B24 and its independently provisioned reproduction.
Other experiments change one evidence surface at a time or bound the
interpretation; we do not pool them as independent samples from a production
population.

\begin{table*}[t]
  \centering
  \caption{Evidence ladder. ``External'' means independently maintained input
  or implementation, not that every adapter or gateway component is external.}
  \label{tab:evidence_ladder}
  \footnotesize
  \setlength{\tabcolsep}{4pt}
  \begin{tabular}{p{0.08\textwidth}p{0.25\textwidth}p{0.39\textwidth}p{0.20\textwidth}}
    \toprule
    Study & Evidence surface & Primary observation & Role in this paper \\
    \midrule
    E-B24 & Primary and separately provisioned repeatability executions;
      16 AWS hosts, 8 gateways,
      3 zones, durable receivers, real effects &
      C2/C3 coexist with missing C4 and 112 committed effects per block; the
      second execution exactly matches the first. &
      Primary causal result and repeatability. \\
    E-B25 & OPA 1.17 policy engines and Toxiproxy 2.12 on one Docker host &
      A post-persistence fault leaves one engine stale: 112/225 effects commit
      and the monitor returns \texttt{VIOLATED}. &
      Cross-implementation replication. \\
    E-B28 & Unmodified external MCPMarkAgent; 10 frozen tasks, 2 repetitions,
      Filesystem and PostgreSQL~\cite{wu2025mcpmark} &
      Normal C4 denies all 31 exposed later calls with zero effects; selective
      C3 allows 344 and commits 145 effects. &
      Closed-loop agent-workload replication. \\
    E-B25/26 & OPA/Toxiproxy and a narrow live multi-role MCP workflow &
      The post-persistence boundary recurs under another policy engine and
      under two model-generated networked actions. &
      Cross-stack execution sanity. \\
    E-B27 & Frozen monitor fault matrix and local SQLite scaling &
      All 14 conservative-classification and integrity rows pass; 33 fresh
      trials span 1,024 gateways and 10,000 effect vectors. &
      Monitor correctness and engineering-cost boundary. \\
    Qualification & E-B17/18/20/22/23: mechanism models, local MCP execution,
      MCPMark traces, and external Keycloak PUSH/POLL~\cite{wu2025mcpmark} &
      The 50\% cost gate and 30\% workload hypothesis fail; real delivery paths
      can reach enforcement when the downstream handoff completes. &
      Cost/workload boundaries and execution sanity checks. \\
    \bottomrule
  \end{tabular}
\end{table*}

\subsection{Frozen Workload and Effect Mapping}

E-B20 runs 30 Filesystem and 21 PostgreSQL MCPMark tasks four times each, using
runs 1--2 for calibration and 3--4 as held out. It retains all 204 task-runs,
including benchmark failures and traces with no mutating calls. The formal
study records 3,436 MCP calls and 921 accepted mutating calls. Its proposed
risk-aware scheduling advantage fails both frozen 30\% outcome gates, so no
risk index, policy, target set, or threshold is retuned.

E-B22--E-B24 inherit the resulting ten target identities and 225 unique
held-out mutating-call vector identities: 112 Filesystem and 113 PostgreSQL.
Raw mutation arguments were not retained. E-B24 therefore performs an
\emph{effect-carrying replay}, not a replay of the original task mutations.
Each frozen identity maps deterministically to one reversible effect: a
canonical file created with flush, atomic rename, and directory
\texttt{fsync}, or a committed row keyed by run, phase, and vector. The gateway
rejects identities outside the frozen allowlist. A remote file manifest and a
PostgreSQL row digest, rather than gateway return values, are the authoritative
oracles.

\subsection{E-B24 Topology and Fault}

E-B24 uses 16 on-demand EC2 instances in one AWS region. Eight physical
gateways are placed 3/3/2 across three availability zones. Keycloak is an
independently maintained open-source identity and access management server.
Separate hosts run Keycloak 26.7.0 and its PostgreSQL database, receiver A,
receiver B, the Filesystem effect service, the PostgreSQL effect
service/database, the workload driver, and an independent fault controller.
Keycloak's experimental
SSF transmitter provides signed CAEP \texttt{session-revoked} SETs over RFC
8935 PUSH. Its implementation queues events durably and also supplies RFC 8936
POLL~\cite{keycloak2026ssf}.

Vector \(i\) is assigned to physical gateway \((i \bmod 8)+1\). Receiver A
owns gateways 1--4 and 113 vectors. Receiver B owns gateways 5--8 and 112
vectors (56 per effect type). The assignment freezes a predicted treatment
boundary of \(112/225=49.78\%\). This fraction is fixed before the outcome run
and is not selected from a pilot.

The fault controller installs a host-level \texttt{nftables} rule on receiver
B that drops only TCP traffic to the eight gateway control endpoints. The rule
does not block Keycloak's PUSH to receiver B, receiver ingress, gateway data
calls, effect traffic, DNS, NTP, or evidence collection. No condition-specific
drop branch exists in receiver, gateway, or workload code.

\subsection{Blocking, Phases, and Outcomes}

The independent unit within an E-B24 execution is a paired cluster block.
Exactly 30 blocks are run with no adaptive stop. Odd blocks execute normal
delivery before the treatment, even blocks reverse that order. Each treatment
is followed by recovery for the same delegation in a fresh effect namespace.
Every phase creates a real Keycloak session, performs OIDC logout, waits for
the frozen gateway pattern, and replays all 225 vectors.

We execute the full 30-block protocol on two separately provisioned 16-host
topologies. Both executions use the same runtime image, freeze manifest, 15/15
order balance, risk index, target membership, policy, and thresholds. The
second execution tests repeatability under the same frozen protocol. We do not
treat the identical, topology-driven counts as 60 stochastic samples from a
production population.

Normal delivery requires HTTP 202, durable receiver state, four receipts per
receiver, and zero admitted calls/effects. In the treatment, receiver A must
reach C4 while receiver B remains at C3 with HTTP 202 and zero receipts.
Exactly receiver B's 112 vectors are expected to commit. Recovery clears the
external fault, waits for receiver B's four receipts, and requires zero effects
in the fresh namespace. All 16 hosts must remain within a frozen 10-ms clock
offset. The primary estimand is the paired difference in stale-admission and
committed-effect fractions. Because the contract predicts exact categorical
counts and the observations have no between-block variation, we report the
effect and replication count rather than a degenerate null-hypothesis test.

\subsection{Cross-Implementation and Live-Agent Replications}

E-B25 preserves the 225 frozen vectors and the post-persistence treatment but
changes the enforcement and fault stack. Two independent OPA 1.17 policy
engines hold the deny state, and Toxiproxy 2.12 disables only the Filesystem
engine's update path after receiver persistence. The data-plane decision path
and both real effect services remain live. Normal, fault, and recovery phases
are judged by the executable C0--C5 monitor. This one-host study tests whether
the boundary depends on our gateway deny implementation; it is not a second
multi-host deployment.

E-B26 changes the execution surface. A live \texttt{gpt-4.1-mini} orchestrator
invokes separate Filesystem and PostgreSQL worker roles, which produce two
networked MCP actions against real effect services. We freeze the resulting
canonical actions before executing active, normal, post-persistence-fault, and
recovery arms, so tool arguments are bit-identical across the paired arms and
model sampling cannot confound the treatment. This is one real agent-tool
workflow, not an independent workload population.

E-B28 removes the frozen-action restriction. It uses the unmodified
MCPMarkAgent at commit \texttt{cd45b7f5}, \texttt{gpt-4.1-mini}, and the exact
ten E-B20 target identities without changing the risk index, policies, or
failed 30\% threshold. Each target runs active, normal C4, and selective C3
arms twice (60 formal task-runs); normal/selective order is balanced. After the
first actual MCP call finishes, but before its result returns to the agent, a
barrier performs a real Keycloak logout and waits for two verified receiver
events. Later normal calls receive gateway denial; later selective calls reach
the official MCP server, so the same agent may retry, switch tools, re-plan, or
stop. Normalized Merkle state and service-side PostgreSQL WAL positions are the
effect oracles. The v2 monitor is constructed over the actual later-call set;
because that set is adaptive, its verdict is an execution check rather than a
claim that the calls were pre-known.

E-B25 and E-B26 use monitor version 1. Their independent effect-oracle records
support the cross-surface comparison, but their C5 labels are not used as
monitor-validation evidence. Version 2 adds immutable evidence and frozen-window
closure. E-B27 freezes its expected verdicts before execution: 14 fresh-state
fault rows cover partial delivery, missing evidence, allow/commit violations,
late receipts, duplicate and scope immutability, restart, window guards, and
missing C0--C2 identifiers. Scaling uses three fresh database trials per cell,
100 assessments after ten warm-ups, gateway counts from 1 to 1,024 at 225
vectors, and vector counts from 1 to 10,000 at eight gateways. No latency SLO
or pass threshold is declared.

\section{Results}
\label{sec:results}

\subsection{E-B24 Reproduces the C2/C3--C4/C5 Gap}

The primary E-B24 execution and the separately provisioned repeatability
execution each pass 30/30 blocks, 90/90 phases, 1,080/1,080 preregistered
assertions, and all clock checks. Table~\ref{tab:eb24_result} reports their
exact match. In each execution, normal delivery denies all 6,750 calls and
produces no authoritative effects. Under selective receiver-to-gateway
non-delivery, receiver B still records HTTP 202 ingress and durable
\texttt{PERSISTED} state in every block, but records no gateway receipt. The
treatment admits 3,360 calls, and every admitted call commits: 1,680 files and
1,680 PostgreSQL rows.

\begin{table*}[t]
  \centering
  \caption{Aggregates from the two separately provisioned E-B24 executions.
  Each execution contains 30 paired blocks and 6,750 calls per phase. Effects
  are counted by independent Filesystem and PostgreSQL oracles; C2--C4 describe
  receiver B. The FS and PG columns give the per-execution decomposition; the
  executions have identical totals.}
  \label{tab:eb24_result}
  \scriptsize
  \begin{tabular}{lrrrrlll}
    \toprule
    Phase & Run 1 allowed/effects & Run 2 allowed/effects & FS & PG & B: C2 & B: C3/C4 \\
    \midrule
    Normal & 0/0 & 0/0 & 0 & 0 & HTTP 202 & persisted/4 receipts \\
    Receiver--gateway fault & 3,360/3,360 & 3,360/3,360 & 1,680 & 1,680 & HTTP 202 & persisted/0 receipts \\
    Recovery & 0/0 & 0/0 & 0 & 0 & retained & retained/4 receipts \\
    \bottomrule
  \end{tabular}
\end{table*}

The paired treatment effect is exactly +49.78 percentage points for both stale
admission and committed effects in every block of both executions. Filesystem
and PostgreSQL each contribute 56 effects per treatment block. Receiver A
remains enforced, so the failure set equals the frozen receiver B assignment
rather than spreading to healthy gateways. The exact failure-set match locates
the result: the observer state \(C2 \land C3\) records the revocation as
delivered and durable while the isolated enforcement partition remains active.

Clearing the fault closes the gap. Receiver B obtains four gateway receipts,
transitions to \texttt{ENFORCED}, and all 225 calls are denied in a fresh
namespace. Both effect oracles return zero in all 60 recovery phases across the
two executions. The normal--fault--recovery sequence rules out an initially
revoked gateway, irreversible corruption, and a workload incapable of producing
effects. The same sequence and totals recur on the separately provisioned
topology. Together, the phases locate the causal boundary after receiver
persistence and before gateway enforcement.

The 49.78\% value should not be read as an incident rate. Half of the physical
gateways are assigned to receiver B and their frozen vectors happen to comprise
112/225 calls. E-B24 measures the consequence conditional on that partition;
it does not estimate how frequently the partition occurs in production.

\subsection{PUSH and POLL Reach Enforcement When the Path Completes}
\label{sec:triangulation}

E-B22 verifies that the observed gap is not caused by an inability to connect
Keycloak to an agent-tool gateway. A real Keycloak logout produces a signed
CAEP \texttt{session-revoked} SET. Under normal delivery, both independent-
process MCP gateways deny all 225 frozen vectors. In the selective condition,
receiver A forwards the SET and its gateway denies 225/225, while receiver B
returns HTTP 202 but withholds its local handoff. Gateway B remains active and
allows 225/225. Keycloak correctly records the acknowledged event as delivered.

E-B23 exercises a second, independently implemented delivery path. Receiver A
uses PUSH while receiver B is configured for native Keycloak RFC 8936 POLL.
Before B polls, A denies and B allows. An authenticated poll then returns one
matching signed SET. After verification and application, both gateways deny all
225 vectors. A second poll acknowledges the event identifier and returns no
event. POLL is therefore a working alternative path into the same deny state.
It is not a fallback for E-B22's already acknowledged PUSH row, which is no
longer pending.

Together, E-B22--E-B24 distinguish a delivery mechanism from a completion
verdict. PUSH and POLL can each reach C4, but transport success alone does not
establish that a receiver-to-gateway handoff occurred. The adapter from a Keycloak
session event to a delegation identifier is self-authored, so these experiments
do not claim Keycloak or gateway-framework conformance beyond the exercised
SSF paths.

\subsection{The Same Boundary Appears Across External Execution Surfaces}
\label{sec:cross_surface}

Table~\ref{tab:replications} reports the two targeted replications. In E-B25,
active OPA policies first allow and commit all 225 effects. Normal delivery then
installs deny state in both policy engines: all 225 calls are denied, no effects
commit, and the stored v1 monitor returns \texttt{COMPLETE\_FOR\_BOUND}. After receiver
persistence, Toxiproxy disables only the Filesystem engine's update path. The
stale engine allows all 112 Filesystem vectors while the PostgreSQL engine
denies its 113 vectors; all 112 allowed calls commit and the monitor returns
\texttt{VIOLATED}. Recovery again denies 225/225 with zero effects. A late
Filesystem receipt can advance stage evidence to C4 but cannot erase the
recorded effect violation.

\begin{table}[t]
  \centering
  \caption{Cross-stack and closed-loop replications. E-B28 denominators include
  only observed post-barrier calls; no-exposure runs are reported separately.}
  \label{tab:replications}
  \scriptsize
  \begin{tabular}{llrrl}
    \toprule
    Study & Arm & Allowed & Effects & Monitor \\
    \midrule
    E-B25 OPA & normal & 0/225 & 0 & v1 complete$^*$ \\
              & fault & 112/225 & 112 & violated \\
              & recovery & 0/225 & 0 & v1 complete$^*$ \\
    E-B26 live MCP & normal & 0/2 & 0 & v1 complete$^*$ \\
                   & fault & 1/2 & 1 & violated \\
                   & recovery & 0/2 & 0 & v1 complete$^*$ \\
    E-B28 closed loop & normal C4 & 0/31 & 0 & 17 complete; 3 none \\
                      & selective C3 & 344/344 & 145 & 18 violated; 2 none \\
                      & recovery C4 & 0/8 & 0 & 2 complete \\
    \bottomrule
  \end{tabular}
  \parbox{0.95\columnwidth}{\scriptsize $^*$The deny/effect records are retained;
  v1 C5 labels are legacy and are not monitor-validation evidence.}
\end{table}

E-B26 first confirms that both frozen actions can produce effects. Under normal
revocation, both networked MCP calls are denied and neither effect commits. In
the post-persistence arm, the Filesystem handoff is missing, its call is
allowed, and the authoritative file effect commits; PostgreSQL is denied. The
monitor returns \texttt{VIOLATED} at C3. Recovery denies both actions, commits
no effects, and receives a stored v1 bounded-completion label. These two studies
show that the observed boundary is not confined to the original gateway deny
code or to identity-only replay. E-B25 remains a one-host engineering
replication, and E-B26 contains only two actions from one model-backed workflow;
neither estimates how often the boundary occurs across deployed agent systems.

E-B28 supplies the broader agent-specific check. All 60/60 formal traces and
2/2 recovery qualifications are valid; 84/84 Keycloak receiver attempts are
authenticated and signature verified. In normal C4, 17/20 runs make 31 later
calls, all are denied, and no effect commits; three runs stop after the first
call and are \texttt{NO\_EXPOSURE}. In selective C3, 18/20 runs make 344 later
calls, all reach the external MCP servers, and 16/20 runs commit 145 effects
(66 Filesystem and 79 PostgreSQL). Two selective runs stop before a later call,
and two make later calls without a state change. Both recovery runs deny all
eight later calls and commit no effect. Thus the boundary survives adaptive
agent feedback on both services, but non-delivery creates an opportunity, not
a guarantee that every sampled trajectory mutates state. Only 4/60 tasks pass
MCPMark's verifier, so this is evidence about live attempt/effect behavior, not
predominantly successful task completion.

\subsection{The Monitor Is Conservative, Persistent, and Measurably Bounded}
\label{sec:monitor_evaluation}

E-B27 evaluates version 2 rather than inferring its correctness from a workload
run. All 14/14 frozen fault rows pass. In particular, C3 without all receipts
is partial, and C4 with a closed window but missing effects is indeterminate.
C5 is reached only after the deadline with all receipts and complete
all-deny/no-commit evidence. Allowed or committed effects are violations; a
late receipt advances stage evidence without clearing that verdict. Identical
duplicates are idempotent; conflicting duplicates and scope changes are
rejected. Reopening both SQLite stores reproduces the verdict byte for byte.

Table~\ref{tab:completion_detector_audit} reports a deterministic, post-hoc
audit against the frozen expected labels for semantic cases F01--F08. A rule
that equates receiver persistence (C3) with completion produces seven false
positives. Requiring all gateway receipts (C4) removes two of them, but still
misclassifies an open observation window, missing effect evidence, and three
effect violations. The bounded C0--C5 verdict has no classification error in
these eight cases. This audit is diagnostic rather than an independent test:
it reuses the registered fault scenarios and excludes the six integrity-only
rows F09--F14.

\begin{table}[t]
  \centering
  \caption{Derived completion-detector audit on E-B27's eight frozen semantic
  cases. Positive means \texttt{COMPLETE\_FOR\_BOUND}.}
  \label{tab:completion_detector_audit}
  \small
  \begin{tabular}{lrrrr}
\toprule
Detector & TP & FP & TN & FN \\
\midrule
C3 persistence as complete & 1 & 7 & 0 & 0 \\
C4 receipts as complete & 1 & 5 & 2 & 0 \\
Bounded C0--C5 monitor & 1 & 0 & 7 & 0 \\
\bottomrule
\end{tabular}

\end{table}

Table~\ref{tab:monitor_scaling} reports descriptive medians across three fresh
database trials per point. The eight-gateway/225-vector point assesses at
1.31\,ms p95. At 1,024 gateways it remains 7.49\,ms p95, but inserting all
1,024 fully synchronized receipts costs 3.81\,s. At 10,000 vectors, assessment
is 25.10\,ms p95 and serial durable effect insertion costs 10.22\,s. These
measurements describe implementation limits rather than a production SLO: v2
favors simple, durable, independently auditable state over batched write
throughput.

\begin{table}[t]
  \centering
  \caption{E-B27 local monitor scaling. Values are medians across three fresh
  stores; assess is p95 over 100 post-warm-up calls.}
  \label{tab:monitor_scaling}
  \scriptsize
  \begin{tabular}{rrr@{\hspace{5pt}}r@{\hspace{5pt}}r}
    \toprule
    Gateways & Vectors & Receipt total & Effect total & Assess p95 \\
             &         & (ms) & (ms) & (ms) \\
    \midrule
    8 & 225 & 3.53 & 68.09 & 1.31 \\
    1,024 & 225 & 3,811.27 & 77.28 & 7.49 \\
    8 & 10,000 & 3.66 & 10,216.50 & 25.10 \\
    \bottomrule
  \end{tabular}
\end{table}

\subsection{Mechanism and Workload Boundaries}

E-B17 tests whether the completion concern disappears when a stronger repair
composition is compared with alternative mechanism families. Ten disclosed
behavioral models share 1,000-gateway traces, 30 paired seeds, and frozen gates
for explicit convergence, stale calls, benign availability, audit evidence,
messages, and bytes. At 20\% selective operator suppression, only push plus
independent repair and matched log polling satisfy every gate. The repair
composition admits 0.877 stale calls per gateway versus 50.549 for log polling,
a 98.27\% reduction, while using 8.001 messages per gateway.

The favorable direction fails the principal cost boundary at 50\% suppression.
Repair admits 1.663 stale calls per gateway but uses 9.501 messages, exceeding
the preregistered 8.1 cap. Matched log polling remains at 6.001 messages. This
failure is retained. It shows that adding repair paths can reduce stale
authority without producing a feasible mechanism at every suppression level.

E-B20 supplies a separate external-validity check and another negative result.
On the frozen MCPMark Filesystem/PostgreSQL attempt-trace population, a
risk-aware allocation derived from E-B19 does not improve either preregistered
outcome by 30\%. Point estimates are slightly unfavorable and small. The
experiment includes all benchmark failures, so it supports neither general
inferiority nor a mechanism contribution. Its role is to prevent the modeled
result from being rewritten as externally reproduced harm reduction.

\section{Implications}
\label{sec:discussion}

\para{Completion requires an enforcement-facing interface}
An issuer API can truthfully report C0, and a transmitter can truthfully report
C1/C2, without knowing C4. A receiver can report C3 without evidence that its
assigned gateways changed state. Systems that expose one of these statuses as
``revocation complete'' collapse distinct distributed facts. A completion API
should instead identify the intended gateway set, collect receipts bound to
the revocation and gateway identity, and state the call population, observation
window, and effect oracle used for C5. The executable monitor provides one
such interface: complete-for-bound, partial-enforcement,
indeterminate-evidence, or violated. Violation evidence is sticky even if a
delayed receipt later arrives. Unlike a transport-success Boolean, these states
preserve the distinction between missing enforcement and missing effect
evidence without implying unbounded future safety. E-B27 also quantifies a
practical cost: durable per-item writes, especially for large receipt sets,
would require batching or a different store before production use.

\para{Agent orchestration needs the completion state}
Revocation is often triggered while an agent task is still live: a user cancels
a delegation, an orchestrator reassigns work, or a risk monitor withdraws tool
authority. A policy that distinguishes these cases needs C4/C5 rather than only
the identity provider's delivery status. High-impact workflows may pause new
tool calls until receipts arrive, route them through a synchronous check, or
terminate affected task branches. The current experiments measure these states
but do not select a universal fail-open or fail-closed policy. E-B28 shows the
operational distinction: after denial, an adaptive agent makes only 31 later
calls across 17 exposed normal runs; under selective non-delivery, it makes 344
allowed calls and commits 145 effects. E-B16 shows why availability makes the
policy choice application dependent.

\para{Repair and acknowledgment are separate design choices}
E-B23 shows that a receiver-initiated POLL stream can deliver the same event
class as PUSH. E-B22 shows why it cannot be assumed to repair a PUSH event that
the transmitter already considers delivered. End-to-end repair therefore
needs a trigger tied to missing enforcement receipts, an independently readable
event source, or a delivery contract in which pending state survives until C4.
Adding a second transport endpoint does not by itself close the semantic gap.

\para{The negative results bound the measurement claim}
E-B17's 50\% cost failure and E-B20's failed risk-aware reproduction do not
support a broader story in which one composed mechanism dominates the field.
The supported principle is narrower: compare mechanisms under a common fault
and cost contract, then verify completion at the gateway and effect boundary.
This interpretation preserves unfavorable cells instead of treating them as
tuning opportunities.

\section{Related Work}
\label{sec:related}

\para{Agent authorization and revocation}
AIP specifies signed notification, retry, degraded-subscriber state, and
status fallback. OAuth-oriented AAP and DAAP address event delivery and
descendant invalidation~\cite{singla2026aip,cruz2026aap,kumar2026daap}. DRP
makes a signed revocation log authoritative and recommends independent
monitoring~\cite{nelson2026drp}. HBHC evaluates heartbeat expiry, gossip,
multipath delivery, and post-revocation tool calls with LLM-backed
swarms~\cite{deochake2026hbhc}. Permit Receipts moves authorization evidence
toward effect boundaries but leaves event distribution to deployment
profiles~\cite{lee2026permit}. These systems motivate the mechanisms we model.
Our object is the cross-stage completion verdict under selective non-delivery.

Recent proposals also combine federation-wide notification, transparency
logs, or enforcement records
~\cite{fane2026opena2aaap,hood2026agtpcert,pidlisnyi2026aps}. This overlap is
why push, polling, gossip, logs, and signed receipts are not claimed as new.
E-B17 is a common-trace behavioral comparison, not a set of source-compatible
reimplementations. Related agent-safety audits show that proxy results need not
license their intended operational interpretation and that configuration can
reverse benchmark verdicts
~\cite{li2026auc0998enoughcandidate,
li2026safetyreproconfigurationconditionalrankinstability}; our frozen contract
applies this measurement caution to revocation.

\para{Continuous security events and revocation distribution}
OAuth token revocation explicitly notes that invalidation can reach some
servers later than others, and that immediate revocation depends on the
resource-server interaction model~\cite{rfc7009}. Our contract turns that
recognized propagation window into a bounded, execution-facing verdict.
Global Token Revocation separately specifies an optional completion
notification after issuer-side token revocation and reauthentication
~\cite{parecki2026globaltokenrevocation}; C0--C5 instead requires receipts
from the declared gateway population and independent observations of bounded
tool effects. CAEP and SSF specify event semantics and PUSH/POLL transport
between a transmitter and receiver~\cite{openid2025caep,openid2025ssf}. Keycloak 26.7
provides the external SSF transmitter used in our execution studies
~\cite{keycloak2026ssf}. Status lists, CRLite, Clubcards, and RevDNS address
compact or scalable revocation distribution
~\cite{w3c2025bitstring,larisch2017crlite,schanck2025clubcards,bhowmick2025revdns},
while Certificate Transparency and CTng address append-only or split-view
evidence~\cite{rfc6962,kong2026ctng}. We do not replace these mechanisms. We
measure the semantic distance from their observable delivery states to an
agent-tool effect boundary.

\para{Agent execution surfaces}
MCP exposes tools to model-backed applications, and A2A represents agents,
tasks, artifacts, and protocol bindings~\cite{mcp2025spec,a2a2026spec}. Neither
base transport specification supplies the complete revocation path studied
here. The recent aiAuthZ gateway demonstrates that off-host, per-call policy
can block agent tool actions, but does not study post-issuance revocation or
completion across a distributed enforcement population
~\cite{kodathala2026aiauthz}. AgentBound combines delegated authorization,
behavioral policies, action contracts, and verifiable action receipts while
preserving revocability, but evaluates per-action governance rather than
cross-gateway revocation completion~\cite{kaul2026agentbound}. Proof of Execution binds authorization, effect,
history, and replay in a runtime-checkable execution object, but evaluates a
single-node prototype rather than revocation propagation across gateways
~\cite{rhodes2026proofexecution}. Notarized Agents moves action evidence to
receiver-signed confidential receipts and a transparency log, while explicitly
retaining receipt suppression as a limitation~\cite{figuera2026notarized}. Our
monitor instead asks whether a frozen enforcement population and effect set
completed after revocation. E-B18 verifies one real local LLM/MCP direction,
E-B20 supplies external MCP workload identities, E-B24 tests a distributed
deployment, E-B25 changes the policy engine, and E-B26 executes live model-
generated networked MCP actions. E-B28 adds the unmodified external MCPMark
agent and natural task loop, but uses one model and one frozen benchmark
commit. These studies do not establish general MCP/A2A interoperability.
E-B27 validates our monitor implementation rather than a third-party
completion-monitor standard.

\section{Limitations}
\label{sec:limitations}

E-B24 estimates the consequence of one frozen partition, not its prevalence.
The 49.78\% effect follows from a topology in which one receiver owns four of
eight gateways and 112 of 225 call vectors. Both executions use one cloud
region, one Keycloak version whose SSF feature is experimental, one
self-authored SET-to-delegation adapter, and one receiver-to-gateway fault. The
effect-carrying vectors preserve external MCPMark identities and service labels
but not the original mutation arguments or natural task semantics.

Implementation breadth is also limited. E-B17 compares ten behavioral models,
not ten production implementations. E-B22/23 exercise external Keycloak
delivery code, but the receivers and gateways are ours. E-B25 adds an
independently maintained policy engine and fault injector, although its adapter
and one-host harness remain ours. E-B26 contains only two frozen actions from
one model and one self-authored workflow. E-B28 uses an unmodified external
agent and natural task loop, but only one model, ten tasks, and two repetitions
on one host. Its task verifier passes 4/60 arms, five comparison arms have no
post-barrier exposure, and two exposed selective arms commit no effect. Its
adaptive call populations therefore establish a bounded execution replication,
not an incident rate or a fixed-trajectory effect ratio. E-B20 likewise has low
task success. These studies do not establish production safety or the
incidence of incomplete revocation across agent platforms.

The completion verdict is bounded to the registered gateways, vectors, and
closed observation window. It cannot exclude future calls, unknown gateways,
or side channels outside that contract. The reference monitor trusts the
supplied C0--C2 evidence identifiers and the integrity of its receiver and
effect stores. E-B27 evaluates a single-host Python/SQLite implementation; its
scaling results do not establish the availability, Byzantine robustness, or
latency of a distributed production monitor. E-B25 and E-B26 used monitor
version 1, so their effect-oracle records support the cross-surface result, but
their bounded-completion labels are not monitor-validation evidence.

\section{Conclusion}
\label{sec:conclusion}

We have shown that a transport acknowledgment or persisted receiver event does
not by itself make revocation complete in a distributed agent system. The
C0--C3 lower bound formalizes this distinction; the executable monitor exposes
it as a bounded systems interface, and E-B27 checks its conservative verdicts,
exposes errors from prefix-only detectors, and measures local scaling limits.
E-B24 reproduces the gap in a primary and a separately provisioned 16-host
execution: selective handoff failure leaves four gateways active and commits
every effect assigned to them, whereas normal delivery and recovery commit
none. OPA/Toxiproxy and a live multi-agent MCP workflow reproduce the boundary
on narrower surfaces. An unmodified external agent further reproduces it in a
closed loop on both Filesystem and PostgreSQL, while also exposing no-effect
and no-exposure repetitions. PUSH, POLL, repair, and logs remain useful
delivery mechanisms, but evaluating revocation completion requires
gateway-enforcement and effect evidence. The contribution is a bounded outcome
contract and the evidence assembled around it.

\bibliographystyle{ACM-Reference-Format}
\bibliography{references}

\clearpage
\appendix

\section{E-B17 Mechanism Disclosure}
\label{app:eb17}

Figure~\ref{fig:eb17} visualizes the principal cost boundary, and
Table~\ref{tab:eb17} gives the underlying confirmatory cells.

\begin{figure*}[t]
  \centering
  \includegraphics[width=0.94\textwidth]{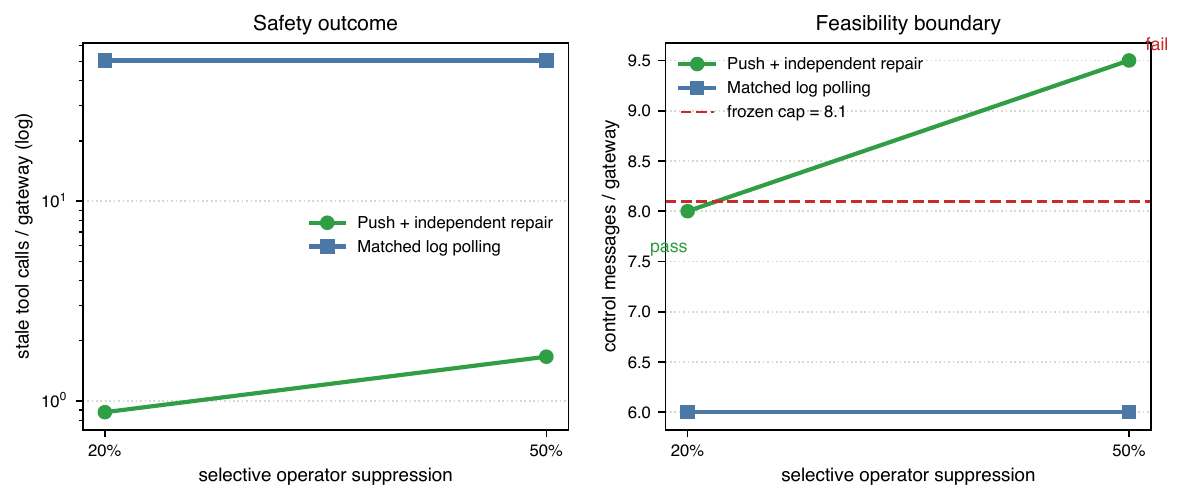}
  \caption{E-B17 mechanism-level boundary. At 20\% selective operator
  suppression, push plus independent repair reduces stale calls and remains
  under the frozen message cap. At 50\%, it crosses the 8.1-message cap.}
  \Description{Two panels compare stale calls and control messages for push
  plus independent repair and matched log polling at 20 and 50 percent
  selective suppression.}
  \label{fig:eb17}
\end{figure*}

\begin{table*}[t]
  \centering
  \caption{Confirmatory E-B17 cells. Availability loss is zero and bytes stay
  below 8,192 for all rows. Feasibility is conjunctive.}
  \label{tab:eb17}
  \small
  \begin{tabular}{llrrrrc}
\toprule
Suppression & Mechanism & Stale calls/gw & Complete & Audit & Msgs/gw & Feasible \\
\midrule
20\% & Push + repair & 0.877 & 1.0 & 1.0 & 8.001 & yes \\
20\% & Matched log polling & 50.549 & 1.0 & 1.0 & 6.001 & yes \\
50\% & Push + repair & 1.663 & 1.0 & 1.0 & 9.501 & no \\
50\% & Matched log polling & 50.549 & 1.0 & 1.0 & 6.001 & yes \\
\bottomrule
\end{tabular}

\end{table*}

The ten rows are AIP-style push/retry/fallback, two named DRP-style log-polling
policies, HBHC-style heartbeat expiry, CAEP/SSF-style streaming, status-list
polling, TTL cache, synchronous fail-open, synchronous fail-closed, and push
plus independent repair. The matched profile imposes common 60-second resource
ceilings. The two DRP policies consequently collapse to the same 10-second
behavior and are one feasible family rather than independent replications.
Table~\ref{tab:model_disclosure} maps every behavioral row to its source and
states the retained abstraction.

\begin{table*}[t]
  \centering
  \caption{Source-to-model disclosure for E-B17. Rows retain selected behavior
  under common traces. They are neither conformance tests nor production
  reimplementations.}
  \label{tab:model_disclosure}
  \footnotesize
  \begin{tabular}{p{0.14\textwidth}p{0.43\textwidth}p{0.34\textwidth}}
\toprule
E-B17 row & Behavior retained in the model & 20\% matched disposition \\
\midrule
AIP -03 & Signed events; 1/2/4-second retries; status fallback~\cite{singla2026aip}. & Fails: completeness 0.8; audit evidence 0. \\
DRP -10/60 & Signed log polling with independently checkable evidence~\cite{nelson2026drp}. & Passes after the disclosed 10-second rescaling. \\
DRP -10/300 & Relaxed polling posture from the same draft~\cite{nelson2026drp}. & Passes, but matches DRP -10/60 and is not an independent replication. \\
HBHC & Published 10/30-second heartbeat with multipath gossip~\cite{deochake2026hbhc}. & Fails: no positive event or audit evidence; 12 messages/gateway exceeds the cap. \\
CAEP/SSF & Signed session-revoked event over a modeled local stream~\cite{openid2025caep,openid2025ssf}. & Fails: completeness 0.8; audit evidence 0. \\
Status list & Signed snapshot polling at the matched 10-second interval. & Fails: completeness 0.8, audit evidence 0, and 750,000 bytes/gateway. \\
TTL cache & Local 60-second credential expiry without a positive event. & Fails: completeness and audit evidence 0. \\
Sync fail-open & Request-time status check that permits on error. & Fails: no positive event or audit evidence; 600 messages/gateway. \\
Sync fail-closed & Request-time status check that denies on error. & Fails: no positive event or audit evidence; 0.2 availability loss; 600 messages/gateway. \\
Push + independent repair & Direct push, bounded peer repair, and a readable log. & Passes: completeness/audit 1.0, zero availability loss, 8.001 messages, and 3,937.2 bytes/gateway. \\
\bottomrule
\end{tabular}

\end{table*}

\section{Supporting Execution Studies}
\label{app:supporting}

\para{E-B18}
Two official MCP Python 1.27.2 gateway servers run as separate local processes.
A local \texttt{llama3.1} model supplies one orchestrator and two worker roles.
All 6/6 runs pass process-isolation, model-output, MCP-call, fault-direction,
and replay gates. Six functional runs do not estimate an LLM population or
multi-vendor behavior.

\para{E-B20}
MCPMark execution covers 204/204 task-runs with 3,436 MCP calls and 921 accepted
mutating calls. Filesystem task success is 4/120 and PostgreSQL success is
24/84. Failed, zero-call, and zero-mutation runs remain in the population. The
paired risk-aware-minus-uniform differences are 0.0875 for the primary harm
metric (95\% interval \([-0.2820,0.5341]\)) and 0.0333 for the unauthorized-
effect count (interval \([0,0.1]\)). Neither preregistered 30\% reduction gate
passes.

\para{E-B22}
Four 225-vector replay groups test normal and selective PUSH delivery. Both
normal gateways deny 225/225. Under selective receiver-local withholding,
gateway A denies 225/225 and gateway B allows 225/225 despite HTTP 202. All 13
registered assertions pass. The guarded tool produces no side effects.

\para{E-B23}
The mixed PUSH/POLL study records 450 replay rows and all 15 registered
assertions pass. The first authenticated POLL request returns one signed SET;
the second acknowledges its identifier and returns none. The measured 140
request-body bytes and 1,374 response-body bytes exclude HTTP/TLS and token-
acquisition overhead.

\para{E-B25}
The independent stack pins OPA 1.17.0 and Toxiproxy 2.12.0 container digests.
An active-state sanity phase allows and commits all 225 frozen vectors. Normal
and recovery each deny 225/225 with zero effects and stored v1
\texttt{COMPLETE\_FOR\_BOUND} labels. The post-persistence fault disables only
the Filesystem OPA update proxy: 112 calls are allowed and commit, 113 are
denied, only the PostgreSQL receipt is present, and the monitor returns
\texttt{VIOLATED}. All 11 registered assertions pass.

\para{E-B26}
The live workflow records exactly three model calls (orchestrator, Filesystem
worker, and PostgreSQL worker) and two networked streamable-HTTP MCP tool
actions. The canonical action arguments are frozen and bit-identical across
arms. Active state commits both real effects; normal and recovery deny both
with no effects; the post-persistence fault allows and commits only the
Filesystem action. All 11 assertions pass.

\para{E-B27}
The version-2 monitor evaluation passes all 14 frozen fault rows and all 33
fresh scaling trials. The matrix covers conservative stage/status
classification, sticky violations, duplicate and scope immutability, window
guards, required evidence identifiers, and exact restart persistence. At the
largest gateway point (1,024 gateways and 225 vectors), median assessment p95
is 7.49\,ms and durable state is 4,824,592 bytes. At the largest vector point
(eight gateways and 10,000 vectors), assessment p95 is 25.10\,ms and durable
state is 6,176,344 bytes. Measurements are descriptive and use no latency pass
threshold.

\section{Secondary Experiments E-B1--E-B16}
\label{app:secondary}

Table~\ref{tab:secondary} groups the remaining component and boundary checks.

\begin{table}[H]
  \centering
  \caption{Secondary checks validate components and boundaries. None replaces
  the confirmatory E-B24 result.}
  \label{tab:secondary}
  \scriptsize
  \setlength{\tabcolsep}{2pt}
  \begin{tabular}{p{0.14\columnwidth}p{0.43\columnwidth}p{0.37\columnwidth}}
    \toprule
    Study & Evidence added & Explicit limit \\
    \midrule
    E-B1--B3 & Suppression boundary, delivery regimes, and mapping from delay
      to stale-authority damage. & Discrete-event models and synthetic rates. \\
    E-B4 & Below-quorum rejection and signed-checkpoint split-view detection. &
      Detection requires eventual cross-observation. \\
    E-B5--B7 & A2A-shaped events, loopback signed transport, and parameter
      sensitivity. & Local serialization/transport, not third-party conformance. \\
    E-B8--B9 & Cached gateway effects and A2A lifecycle-binding cells. &
      Gateway and compatibility models. \\
    E-B10--B12 & Gateway/tool/log TCP paths, envelopes, dropped push, delayed
      log, partition heal, and restart. & Scripted single-host faults. \\
    E-B13--B15 & LangGraph/CrewAI trace bridge, separate processes, official
      MCP typing, and a compiled LangGraph workflow. & Local execution, not
      production interoperability. \\
    E-B16 & Fail-closed, fail-open, TTL-cache, and repair under status outage. &
      Local status service. Policy choice remains application dependent. \\
    \bottomrule
  \end{tabular}
\end{table}

E-B4's prototype uses an M-of-N individual-signature quorum certificate and
an in-process, transparency-inspired log; it tests rejection and
cross-observation only.

\section{Availability, Ethics, and Disclosures}
\label{app:disclosures}

\para{Data and artifact availability}
The artifact contains protocols, frozen manifests, raw and aggregate
outputs, validators, infrastructure definitions, replay scripts, the
executable C0--C5 monitor, and its indistinguishability witness. Versioned
evidence hashes are reported above. The artifact bundle is available at
\url{https://anonymous.4open.science/r/dai26-review-artifact-be90ff8a-9E12/README.md}.
The bundle includes a per-file SHA-256 manifest and an executable verifier.

\para{Ethics}
The experiments use synthetic identities, public benchmark tasks, isolated
cloud resources, and reversible effects. They involve no human participants,
personal data, or production user systems.


\para{AI-tool disclosure}
Generative AI tools assisted with code and manuscript editing. The authors
designed the experiments and verified the code, raw evidence, analyses,
citations, and claims.

\end{document}